\documentclass[11pt]{article}

\usepackage{graphicx,subfigure}
\usepackage{epsfig}
\usepackage{hyperref}
\usepackage{amsmath}
\usepackage{amssymb}
\usepackage{algorithm}
\usepackage{algorithmic}
\usepackage{url}
\usepackage{enumerate}
\usepackage{amsfonts}
\usepackage{boxedminipage}
\usepackage{xcolor}
 \usepackage{framed}  
\usepackage{rotating}
\usepackage{array}
\usepackage{multirow}
\usepackage{color}
\usepackage{tikz}
\usepackage{pgfplots}
\usepackage{mathabx}
\usepackage{tabularx,ragged2e,booktabs,caption}

{
\begin{center}
\begin{boxedminipage}{0.8\linewidth}
\begin{center}
\textbf{\texttt{#1}}
\end{center}
\rm
\begin{tabbing}
....\=...\=...\=...\=...\=  \+ \kill
} %
{\end{tabbing} 
\end{boxedminipage} \end{center} 
}

{
\vspace{0.01cm}
\begin{center}
\begin{boxedminipage}{0.8\linewidth}
\begin{center}
\textbf{\texttt{#1}}
\end{center}
} %
{
\end{boxedminipage} \end{center} \vspace{0.3cm}
}

\newtheorem{lemma}{Lemma}
\newtheorem{corollary}{Corollary}
\newtheorem{theorem}{Theorem}

\newtheorem{example}{Example}

\newtheorem{remark}{Remark}

\newcommand{\bE}{\mathbb{E}}
\newcommand{\reals}{\mathbb{R}}

\newcommand{\spn}{\mathrm{span}}

\newcommand{\cY}{\mathcal{Y}}
\newcommand{\cX}{\mathcal{X}}
\newcommand{\cD}{\mathcal{D}}

\newcommand{\cB}{\mathcal{B}}

\newcommand{\cA}{\mathcal{A}}

\newcommand{\cW}{\mathcal{W}}

\newenvironment{proof}{\par\noindent{\bf Proof\ }}{\hfill\BlackBox\\[2mm]}
\newcommand{\BlackBox}{\rule{1.5ex}{1.5ex}}

\makeatletter
\def\moverlay{\mathpalette\mov@rlay}
\def\mov@rlay#1#2{\leavevmode\vtop{%
   \baselineskip\z@skip \lineskiplimit-\maxdimen
   \ialign{\hfil$\m@th#1##$\hfil\cr#2\crcr}}}
\newcommand{\charfusion}[3][\mathord]{
    #1{\ifx#1\mathop\vphantom{#2}\fi
        \mathpalette\mov@rlay{#2\cr#3}
      }
    \ifx#1\mathop\expandafter\displaylimits\fi}
\makeatother

\newcommand{\hL}{\hat{L}}

\newcommand{\inter}{\mathrm{int}}

\newcommand{\tr}{\mathrm{tr}}

\DeclareMathOperator*{\argmin}{arg\,min}

\renewcommand{\eqref}[1]{Equation~(\ref{#1})}
\newcommand{\exref}[1]{Example~(\ref{#1})}

\newcommand{\secref}[1]{Section~\ref{#1}}

\newcommand{\thmref}[1]{Theorem~\ref{#1}}
\newcommand{\lemref}[1]{Lemma~\ref{#1}}

\newcommand{\corref}[1]{Corollary~\ref{#1}}

\newcommand{\handout}[5]{
   \renewcommand{\thepage}{#1-\arabic{page}}
   \noindent
   \begin{center}
   \framebox{
      \vbox{
    \hbox to 5.78in { {\bf (67577) Introduction to Machine Learning}
         \hfill #2 }
       \vspace{4mm}
       \hbox to 5.78in { {\Large \hfill #5  \hfill} }
       \vspace{2mm}
       \hbox to 5.78in { {\it #3 \hfill #4} }
      }
   }
   \end{center}
   \vspace*{4mm}
}

\newcommand{\tw}{\tilde{w}}

\newcommand{\tx}{\tilde{x}}

\newcommand{\hkappa}{\hat{\kappa}}
\newcommand{\hrho}{\hat{\rho}}
\newcommand{\hmu}{\hat{\mu}}
\newcommand{\hC}{\hat{C}}

\newcommand{\hDelta}{\hat{\Delta}}
\newcommand{\hell}{\hat{\ell}}
\newcommand{\hw}{\hat{w}}

\author{Alon Gonen\footnote{School of Computer Science, The Hebrew University, Jerusalem, Israel} \and Shai Shalev-Shwartz\footnote{School of Computer Science, The Hebrew University, Jerusalem, Israel}}

\title{Average Stability is Invariant to Data Preconditioning. Implications to Exp-concave Empirical Risk Minimization}
\begin{document}
\maketitle

\begin{abstract}
We show that the average stability notion introduced by
\cite{kearns1999algorithmic, bousquet2002stability} is invariant to
data preconditioning, for a wide class of generalized linear models
that includes most of the known exp-concave losses. In other words, when analyzing the stability rate of a given algorithm, we may assume the optimal preconditioning of the data. This implies that, at least from a statistical perspective, explicit regularization is  not required in order to compensate for ill-conditioned data, which stands in contrast to a widely common approach that includes a regularization for analyzing the sample complexity of generalized linear models. Several important implications of our findings include: a) We demonstrate that the excess risk of empirical risk minimization (ERM) is controlled by the preconditioned stability rate. This immediately yields a relatively short and elegant proof for the fast rates attained by ERM in our context. b) We strengthen the recent bounds of \cite{hardt2015train} on the stability rate of the Stochastic Gradient Descent algorithm. 
\end{abstract}
\section{Introduction} \label{sec:intro}
Central to statistical learning theory is the notion of (algorithmic)
\emph{stability}. Since being introduced by
\cite{bousquet2002stability},  deep connections between the
\emph{generalization} ability and the algorithmic stability of a
learning algorithm have been established. It was shown by
\cite{shalev2010learnability,mukherjee2006learning} that stability
characterizes learnability. Furthermore, in expectation, some notion
of stability is exactly equal to the generalization error of an algorithm (namely, to the gap between true loss and train loss). 

For generalized linear learning problems, a prominent geometric property which upper bounds the stability rate is the condition number of the loss function. While uniform convergence bounds (\cite{shalev2014understanding}[Chapter 4]) mostly yield bounds that scale with $1/\sqrt{n}$, where $n$ is the size of the sample, \emph{well-conditioned} problems admit faster (stability) rates that scale linearly with $1/n$. The caveat is that typical (large-scale) machine learning problems are ill-conditioned. While we defer the precise definition of the \emph{condition number} to the next part, let us mention that the condition number is controlled by two related quantities corresponding to both the choice of the loss function and the choice of the coordinate system. In a nutshell, our paper establishes the following result: 
\begin{center}
The average stability of ERM is invariant to the choice of the coordinate system. 
\end{center}
While this observation admits a one-line proof, it has far-reaching implications. In particular, in this paper we use this observation to establish fast rates for empirical risk minimization. 

The rest of the paper is organized as follows. In \secref{sec:preliminaries} we define the setting and proceed to provide basic definitions and results in stability analysis. In \secref{sec:precondStability} we state and prove our main result. \secref{sec:implications} discusses the implications to linear regression as well as improved bounds on the stability of SGD. Related work is discussed in \secref{sec:related}.

\section{Preliminaries} \label{sec:preliminaries}
\subsection{Setup} \label{sec:setup}
We consider the problem of minimizing the \emph{risk} associated with \emph{generalized linear model}:
\begin{equation} \label{eq:risk}
\min_{w \in \cW} L(w):= \bE_{(x,y) \sim \cD} [\phi_y(w^\top x)]~.
\end{equation}
Here, both the \emph{domain} $\cW$ and the \emph{instance space} $\cX$ are assumed to be compact and convex subsets of $\reals^d$. We denote by $\cD$ an arbitrary probability distribution defined over $\cX \times \cY$. Each element $y$ in the \emph{label set} $\cY$ induces a twice differentiable\footnote{As we do not require smoothness of the loss function, our results can easily be extended to continuous but non-differentiable functions.} \emph{loss function} of the form $\phi_{y}: \{w^\top x:~w \in \cW,~x \in \cX\} \rightarrow \reals_+$. We make the following assumptions on the loss function:\\

\noindent \textbf{(A1)} For each $y \in \cY$, $\phi_y$ is $\rho$-Lipschitz, i.e., $|\phi_y'(z)| \le \rho$ for all $z$.\\
 \textbf{(A2)} For each $y \in \cY$, $\phi_y$ is $\alpha$-strongly convex, i.e., $\phi_y''(z) \ge \alpha$ for all $z$.\\

\noindent Our main example is the following formulation of \emph{linear regression} (\cite{orabona2012beyond}).
\begin{example} \label{ex:regBoundedPred} \textbf{(Linear Regression:) }
Let $\cX$ be any compact and convex subset of $\reals^d$ and $\cY$ be an interval of the form $[-Y,Y]$. The domain $\cW$ is given by 
\begin{equation*}
\cW = \{w \in \reals^d: (\forall x \in \cX)~|w^\top x|  \le Y\}~.
\end{equation*}
For all $y \in \cY$, let $\phi_y$ be the square loss, $\phi_y(z) = \frac{1}{2} (z-y)^2$. Note that for any $y \in \cY$ and $z \in \{w^\top x: w \in \cW,~x \in \cX\}$, 
$$
|\phi_y'(z)| = \frac{1}{2} |2(z-y)| \le |z|+|y| \le 2Y,~~~\|\phi_y''(z)\| = 1~.
$$
Hence, the assumptions \textbf{(A1-2)} are satisfied with $\rho = 2Y$ and $\alpha=1$.
\end{example}
More generally, our setting captures all known \emph{exp-concave} functions (\cite{kivinen1999averaging}). A twice-continuously differentiable function $f:\cW \rightarrow \reals$ is said to be $\bar{\alpha}$-exp-concave if $\nabla^2 f(w) \succeq \bar{\alpha} \nabla f(w) \nabla f(w)^\top$ for all $w \in \cW$. 
\begin{lemma} \label{lem:expConcave}
Consider a risk of the form (\ref{eq:risk}) that satisfies the assumptions \textbf{(A1-2)}. Then, for any $(x,y) \in \cX \times \cY$, the function $w \in \cW \mapsto \phi_y(w^\top x)$ is $\alpha/\rho^2$-exp concave.
\end{lemma}
\begin{proof}
Fix a pair $(x,y) \in \cX \times \cY$. The gradient and the Hessian of the map $\ell(w) = \phi_y(w^\top x)$ are given by
\begin{equation} \label{eq:gradHessianExpConvave}
\nabla \ell(w) = \phi'(w^\top x) x,~~\nabla^2 \ell(w) = \phi''(w^\top x) xx^\top~.
\end{equation}
By assumption $|\phi'(w^\top x)| \le \rho$ and $\phi''(w^\top x) \ge \alpha$, hence $\ell$ is $\alpha/\rho^2$-exp concave.
\end{proof}

A learning algorithm $\cA$ receives as an input a training sequence (a.k.a. sample) of $n$ i.i.d. pairs, $S=((x_i,y_i))_{i=1}^n \sim \cD^n$, and outputs a predictor, $\cA(S) \in \cW$. The \emph{empirical risk function}, $\hL:\cW \rightarrow \reals$, is defined as
\begin{equation} \label{eq:empiricalRisk}
\hL_S(w) = \hL(w) =  \frac{1}{n} \sum_{i=1} ^n \underbrace{\phi_{y_i} (w^\top x_i)}_{:=\hell_i(w)}~.
\end{equation}
In this paper we focus on the ERM algorithm, whose output is a minimizer of the empirical risk.\footnote{The compactness of $\cW$ implies that both the true and the empirical risks admit minimizers.} We denote the output of the ERM by $\hw(S)$, or simply $\hw$ when $S$ is understood from the context. 
The generalization error and the excess risk of $\hw$ are defined by $L(\hw)- \hL(w)$ and $L(\hw) - L(w^\star)$, respectively. For ERM, it is immediate that any upper bound on the generalization error translates into the same bound on the excess risk.
\begin{remark} \label{rem:almostERM}
While we mostly focus on exact ERM, it should be emphasized that our results are easily extended to any algorithm that approximately minimizes the empirical risk. The formulation of \lemref{lem:redStable} below highlights this idea.
\end{remark}

\subsection{Stability}
In this section we review basic definitions and results on stability. For completeness, we also provide proofs of the stated results.

Let $S = ((x_i,y_i))_{i=1}^n$ be a training sequence. For every $i \in [n]$, let $\hw_i$ be a minimizer of the risk w.r.t. $S \setminus \{(x_i,y_i)\}$, namely, 
\[
\hw_i \in \argmin_{w \in \cW} \frac{1}{n-1} \sum_{j \neq i} \hell_j(w) ~.
\]
The \emph{average stability} of ERM is defined as
\begin{equation} \label{eq:avgStable}
\Delta(S,\cW) = \frac{1}{n} \sum_{i=1}^n (\hell_i(\hw_i) - \hell_i(\hw)) ~.
\end{equation}
We omit the dependency on $\cW$ when it is clear from the context. 
The next lemma shows that the expected generalization error of the ERM is equal to the expected average stability.
\begin{lemma} \label{lem:redStable}
Let $\cA$ be a possibly randomized algorithm and denote by $\hw$ its output. The generalization error of $\cA$ satisfies
\begin{equation} \label{eq:redStable1}
\bE_{S \sim \cD^{n-1}} [L(\hw)-\hL(\hw)] =  \bE_{S \sim \cD^n} [\Delta(S)]~.
\end{equation}
Furthermore, if $\cA$ satisfies, for every sample $S$,  $\bE[\hL(\hw)] \le \min_{w \in \cW} \hL(w)+\epsilon$, where the expectation is with respect to $\cA$'s own randomization, then 
the excess risk of $\cA$ is bounded by
\begin{equation} \label{eq:redStable2}
\bE_{S \sim \cD^{n-1}}[L(\hw)-L(w^\star)]  \le  \bE_{S \sim \cD^n} [\Delta(S)]+\epsilon~.
\end{equation}
\end{lemma}
\begin{proof} 
Since $\hw_i$ does not depend on the i.i.d. pair $(x_i,y_i)$,
\[
\bE_{S \sim \cD^n} [\hell_i(\hw_i)] = \bE_{S \sim \cD^{n-1}}[L(\hw(S))]~,i=1,\ldots,n~.
\]
By linearity of expectation, we obtain
\[
\bE_{S \sim \cD^n} \left [\frac{1}{n} \sum_{i=1}^n\hell_i(\hw_i) \right] = \bE_{S \sim \cD^{n-1}}[L(\hw(S))]~.
\]
Therefore,
\[
\bE[\Delta(S)] = \bE_{S \sim \cD^n} \left [\frac{1}{n}
  \sum_{i=1}^n\hell_i(\hw_i) \right] - \bE_{S \sim \cD^n} \left
  [\frac{1}{n} \sum_{i=1}^n\hell_i(\hw) \right] = \bE_{S \sim
  \cD^{n-1}} [L (\hw)] - \bE_{S \sim \cD^n}[\hL(\hw)]~.
\]
This establishes the first claim. 

Next, by assumption, for every $S$, $\bE[\hL (\hw)] \le \hL(w^\star)+\epsilon$. Hence, 
\[
\bE_{S \sim \cD^n} [\hL(\hw)]  \le \bE_{S \sim \cD^n} [\hL(w^\star)]+\epsilon = L(w^\star)+\epsilon~.
\]
Combining this inequality with the first claim, concludes the proof.
\end{proof}


\subsubsection{Stability of Well-conditioned Objectives}  \label{sec:wellCondStable}
\lemref{lem:redStable} motivates us to derive an upper bound on the average stability. A key quantity that governs $\Delta(S)$ is the condition number of the objective. We next provide exact definitions and discuss this relation.

Fix a training sequence $S$. We denote the empirical covariance matrix by 
$$
\hC := \hC(S)= \frac{1}{n} \sum_{i=1}^n x_i x_i^\top~.
$$
The (average) empirical condition number of $\hC$ is defined as
\[
\kappa(\hC) = \frac{\tr(\hC)}{\lambda_{\min}(\hC)}~,
\]
where $\tr(\hC)$ is the trace of $\hC$ and $\lambda_{\min}(\hC)$ is the smallest nonzero eigenvalue of $\hC$. We define the functional condition number as the ratio between the squared Lipschitz parameter and the strong convexity parameter:
\[
\kappa(\phi)= \frac{\rho^2}{\alpha}~.
\]
Finally, we define the condition number of the objective as the product between the empirical and the functional condition number:
\[
\kappa = \kappa(\hC) \kappa (\phi)~.
\]
\begin{lemma} \label{lem:stabilityKappa}
For every training sequence $S$,
\begin{equation} \label{eq:stabilityKappa}
\Delta(S) \le  \frac{2\kappa}{n}  = \frac{2 \kappa(\hC) \kappa (\phi)}{n}  =  \frac{2\rho^2}{\alpha \, n} \kappa(\hC) ~.
\end{equation}
\end{lemma}
To the best of our knowledge, this result has only been proved in the context of regularized loss minimization (e.g., the bound on the uniform stability in \cite{shalev2014understanding}[Corollary 13.6]). Inspecting the proofs, one can notice that the role of regularization is merely to ensure the strong convexity of the objective. This simple observation is crucial for our development. \\

\begin{proof} \textbf{(of \lemref{lem:stabilityKappa})}
We first assume that $\hC$ is of full rank. Note that for all $w$, the Hessian of $\hL$ at $w$ is given by
\begin{equation} \label{eq:hessianGLM}
\nabla ^2 \hL(w) = \frac{1}{n} \sum_{i=1}^n \phi''(w^\top x_i) x_i x_i^\top \succeq   \frac{1}{n} \sum_{i=1}^n  \alpha \,x_i x_i^\top = \alpha\, \hC~.
\end{equation}
In particular, $\hL$ is strongly convex and $\hw$ is uniquely defined. Denote the strong convexity parameter of $\hL$ by $\hmu$. We also denote the Lipschitz parameter of each $\hell_i$ by $\hrho_i$ and define $\hrho^2=\frac{1}{n} \sum_{i=1}^n \hrho_i^2$.  We will shortly derive upper and lower bounds on these parameters, but first let us relate them to the average stability.

Fix some $i \in [n]$ and let $\hDelta_i = \hell_i(\hw_i)-\hell_i(\hw)$ (we do not assume that $\hw_i$ is uniquely defined). The $\hrho_i$-Lipschitzness of $\hell_i$ yields the bound 
$$
\hDelta_i \le \hrho_i \|\hw_i-\hw\|~.
$$
The $\hmu$-strong convexity of $\hL$ implies (e.g. using \cite{shalev2011online}[Lemma 2.8]) that 
$$
\frac{\hmu}{2} \|\hw_i-\hw\|^2 \le \hL(\hw_i) - \hL(\hw)~.
$$ 
On the other hand, since $\hw_i$ minimizes the risk over $S \setminus \{(x_i,y_i)\}$, we have that 
$$
\hL(\hw_i) - \hL(\hw) = \frac{\sum_{j \neq i} (\hat{\ell}_j(\hw_i) - \hat{\ell}_j(\hw))}{n} + \frac{\hat{\ell}_i(\hw_i) - \hat{\ell}_i(\hw)}{n} \le 0 + \frac{\hDelta_i}{n}~.
$$
Combining the bounds, we conclude the following bound for every $i \in [n]$:
\[
\hDelta_i^2 \le \hrho_i^2 \|\hw_i-\hw\|^2 \le \frac{2\hrho_i^2}{\hmu} (\hL(\hw_i) - \hL(\hw)) \le \frac{2\hrho_i^2}{n\hmu} \hDelta_i ~.
\]
Dividing by $\hDelta_i$ (we may assume w.l.o.g. that $\hDelta_i>0$), we obtain
\begin{equation} \label{eq:uniHafirot}
\hDelta_i \le \frac{2 \hrho_i^2}{n\hmu} ~.
\end{equation}
Let us remark that at this point, we can deduce a bound of $\max_{i \in [n]} \frac{2 \hrho_i^2}{n\hmu}$ on the uniform stability. This matches the bound in \cite{shalev2014understanding}[Corollary 13.6]. We next proceed to establish the claimed bound on the average stability.

By averaging (\ref{eq:uniHafirot}) over $i=1,\ldots,n$ , we obtain
\begin{equation} \label{eq:hafirot1}
\hDelta = \frac{1}{n} \sum_{i=1} ^n \hDelta_i \le \left( \frac{1}{n} \sum_{i=1}^n \hrho_i^2 \right)  \frac{2}{n\hmu}  = \frac{2 \hrho^2}{n \hmu}~.
\end{equation}
It remains to derive bounds on $\hrho$ and $\hmu$. Note that
\[
\|\nabla \hell_i(w)\|^2 = \|\phi'(w^\top x_i) x_i\|^2 \le \rho^2 \|x_i\|^2= \rho^2 \tr(x_ix_i^\top)~.
\]
Hence, $\hrho_i ^2 \le \rho^2  \tr(x_ix_i^\top)$. By averaging, we obtain that $\hrho^2 \le \rho^2 \tr(\hC)$.
Next, using (\ref{eq:hessianGLM}) we obtain that $\hmu \ge \alpha \lambda_d(\hC)$. By substituting the bounds on $\hrho^2$ and $\hmu$ in (\ref{eq:hafirot1}), we conclude the desired bound.

Note that if $\hC$ is not of full rank, we can replace each vector $x \in \reals^d$ with $U^\top x$, where the columns of $U$ form an orthonormal basis for $\spn(\{x_1,\ldots,x_n\})$, without affecting $\hDelta, \hDelta_1,\ldots,\hDelta_n$ (this modification is only for the sake of the analysis). As a result, the new covariance matrix is of full rank and its eigenvalues are $\lambda_1(\hC),\ldots,\lambda_{\min}(\hC)$. Repeating the above arguments, we conclude the proof.
\end{proof}
Let us specify the bound to linear regression as formulated in \exref{ex:regBoundedPred}. As $\alpha=1$ and $\rho = 2Y$, the functional condition number is $4Y^2$. Hence, the average stability is bounded by 
\begin{equation}  \label{eq:regBoundedPredBefore}
\Delta(S) \le \frac{4Y^2}{n} \hkappa(\hC)~.
\end{equation}
Using \lemref{lem:redStable} we deduce the same bound on the excess risk. The weakness of this bound stems from the fact that empirically, the empirical condition number tends to be huge (e.g., see the empirical study in \cite{gonen2016solving}). 

In the next section we show that the (dependence on the) empirical condition number associated with our arbitrary choice of coordinate system can be replaced by the empirical condition number obtained by an optimal preconditioning. 

\section{Preconditioned Stability} \label{sec:precondStability}
We are now in position to describe our main result. Let $P$ be a (symmetric) positive definite matrix, $S_P$ be the training set obtained by replacing every $x_j$ with $\tx_j = P^{-1/2} x_j$, and $\cW_P = P^{1/2} \cW$. We call $P^{-1/2}$ a \emph{preconditioner}. Recall the definition of average stability from \eqref{eq:avgStable}. Our main theorem is:
\begin{theorem} \label{thm:mainInvariant}
For any training sequence $S$ and positive definite matrix $P$,  
\[
\Delta(S_P,\cW_P) = \Delta(S,\cW) ~.
\]
In words, the average stability is invariant to the choice of the coordinate system.
\end{theorem}
\begin{proof}
The crucial observation is that the empirical risk minimization with respect to $S_P$ over the domain $\cW_P$ is equivalent to the ERM w.r.t. $S$ over the domain $\cW$ in the following sense. For any pair $(w, \tw = P^{1/2}w) \in \cW \times \cW_P$ and any $j \in [n]$, the prediction $(\tw)^\top \tx_j$  is equal to the prediction $w^\top x_j$ . Therefore, the empirical risks $\hL_{S_p}(\tw)$ and $\hL_S(w)$ are equal. By associating the corresponding minimizers of the empirical risk (i.e., $\hw$ is associated with $P^{1/2} \hw$ and for any $i \in [n]$, $\hw_i$ is associated with $P^{1/2} \hw_i$), we conclude our proof.
\end{proof}
\thmref{thm:mainInvariant} tells us that we can analyze the stability of $S_P$ instead of the stability of $S$. Crucially, this is true for every $P$, even one that is chosen based on $S$. Therefore, the expected suboptimality is upper bounded by the expected value of the quantity, $\inf_{P \succ 0} \Delta(S_P,\cW_P)$, which we refer to as the \emph{preconditioned average stability}. Equipped with this observation, we next choose $P$ that leads to a minimal condition number, and consequently obtain a tighter bound on the excess risk.

Note that for every $P \succ 0$, the empirical covariance matrix that corresponds to the preconditioned training sequence, $S_P$, is 
\[
\frac{1}{n} \sum_{i=1}^n (P^{-1/2} x_i) (P^{-1/2} x_i)^\top = P^{-1/2} \left(\frac{1}{n} \sum_{i=1}^n  x_i x_i^\top \right) P^{-1/2} = P^{-1/2} \hC P^{-1/2} ~.
 \]
When $\hC$ is of full rank, by choosing $P=\hC$, we obtain that 
$$
\kappa(\underbrace{P^{-1/2} \hC P^{-1/2}}_{I})= \frac{\tr(I)}{\lambda_{\min} (I)} = d~.
$$
If $\hC$ is not of full rank, we can add arbitrary ``noise'' in directions that do not lie in the column space of $\hC$. For example, by choosing $P = \hC+\delta(I-\hC \hC^\dagger)$, (where $\delta$ can be any positive scalar), we obtain that  $\kappa(P^{-1/2} \hC P^{-1/2})=\mathrm{rank}(\hC) \le d$. It is easy to see that in both cases, we obtain the minimal value of $\kappa(P^{-1/2} \hC P^{-1/2})$ over all matrices $P \succ 0$. Combining this bound with \lemref{lem:redStable} and \lemref{lem:stabilityKappa}, we arrive at the following conclusion. 
\begin{corollary} \label{cor:main}
Consider the optimization problem (\ref{eq:risk}), where for all $y \in \cY$, $\phi_y$ is $\rho$-Lipschitz and $\alpha$-strongly convex. The expected excess risk of empirical risk minimization is bounded by
\[
\bE_{S \sim \cD^{n-1}} [L(\hw) - L(w^\star)] \le \bE_{S \sim \cD^n} [\Delta(S)] 
= \bE_{S \sim \cD^n} [\inf_{P \succ 0} \Delta(S_P)] 
\le \frac{2 \rho^2\,d }{\alpha n} ~.
\]
\end{corollary}
\begin{remark} \label{rem:anyAlg}
Note that the theorem holds for any algorithm; We
only use the fact that the prediction itself is invariant to
preconditioning (for any algorithm). 
\end{remark}
Using \lemref{lem:redStable}, we can deduce similar bound holds for approximate ERM.

\section{Some  Implications} \label{sec:implications}
\subsection{Linear Regression}
We start by specifying our bounds to linear regression (\exref{ex:regBoundedPred}). 
\begin{corollary} \textbf{(Linear Regression)}\label{cor:linreg}
Consider linear regression as formulated in \exref{ex:regBoundedPred}. The expected excess risk of ERM is bounded by
\[
\bE_{S \sim \cD^{n-1}} [L(\hw) - L(w^\star)]  \le \Delta(S) \le \frac{4Y ^2d}{n}~.
\]
\end{corollary}
Comparing the bounds in (\ref{eq:regBoundedPredBefore}) and \corref{cor:linreg}, we see that the dependence on $\hkappa(\hC)$ is replaced by the optimal empirical condition number, $\hkappa(I)=d$. As we mentioned above, this gap tends to be huge in practice.

As we discuss in \secref{sec:related}, standard bounds for this setting depend on the geometry of $\cX$ and $\cW$. On the contrary, it follows from the generalized Cauchy-Schwarz inequality that for any choice of a norm $\|\cdot\|$ on $\reals^d$, our bound applies to the sets\footnote{In fact, under mild additional assumptions on $\cX$, any two sets $\cX$ and $\cW$ that satisfy our assumptions can be presented in this way. Assume that $\cX$ is symmetric (i.e., $x \in \cX$ iff $-x \in \cX$) and $0 \in \inter(\cX)$. Then it is known (\cite{conway2013course}) that $\cX$ induces a norm on $\reals^d$ through the Minkowsky functional
\[
\|x\| := p(x) = \inf\,\{t \in \reals: x \in tB\}~.
\]
It is immediate that the closed unit ball $\{x: \|x\| \le 1\}$ is $\cX$ itself. Therefore, the dual norm is simply the support function of $\cX$
\[
\|w\|^\star = \max_{x \in \cX} w^\top x ~.
\]
It follows that $\cX$ and $\cW$ can be described as in (\ref{eq:polar}).} 
\begin{equation} \label{eq:polar}
\cX = \cB_{\|\cdot\|} = \{x \in \reals^d: \|x\| \le 1\},~~\cW = Y \cB_{\|\cdot\|^\star} :=\{w \in \reals^d: \|w\|^\star \le Y\}
\end{equation}

\subsection{The Average Stability of Stochastic Gradient Descent} \label{sec:SGD}
One of the most widely used algorithms in machine learning is Stochastic Gradient Descent (SGD). Besides its computational simplicity, its popularity stems also from its generalization abilities (\cite{shalev2014understanding}[Section 14.5]). Recently, \cite{hardt2015train} studied the (uniform) stability of SGD in various settings. Following our notation, theorem 3.9 of their paper implies a bound of $\max_i\frac{2\hat{\rho}_i^2}{\gamma n}$ on the uniform stability, where $\gamma$ is the strong convexity of the entire objective, and for any $i \in [n]$, $\hrho_i$ is the Lipschitz parameter of $\hell_i$. As the proof of \lemref{lem:stabilityKappa} reveals, $\gamma$ can be bounded by $\alpha \hkappa(\hC)$ and $\hrho_i$ is at most $\hrho^2 \|x_i\|^2$. In particular, the bound depends on the choice of the coordinate system.

As implied by \cite{bubeck2015convex}[Theorem 3.2], SGD can be viewed in our context as an (approximate) ERM. Hence, the average stability of SGD is invariant to the choice of the coordinate system and the stability rate of SGD is bounded as in \corref{cor:main}.


\section{Related Work} \label{sec:related}
\subsection{Slower rates} \label{sec:slow}
One of the most direct techniques for establishing bounds on the
excess risk is via analyzing the Rademacher complexity
(\cite{bartlett2003rademacher}) of the associated class of
predictors. In our setting, these techniques have been employed by
\cite{kakade2009complexity} to establish bounds of order $1/\sqrt{n}$
on the generalization error of ERM.  We refer to these rates as slower
due to the inferior dependence on the sample size $n$. 

\subsection{Dependence on Norm}
Note that since both the uniform and the average stability of ERM are bounded above by its generalization error (\cite{shalev2010learnability}), the bounds of \cite{kakade2009complexity} translate into bounds on the average stability.

Unlike our fast rates, the exact bounds depend on the geometry of the
set $\cX$ and $\cW$. For example: a) If both $\cX$ and $\cW$ are the
Euclidean unit ball in $\reals^d$, then the obtained bound scales with
$1/\sqrt{n}$. b) If $\cX$ is the unit $\ell_\infty$-ball and $\cW$
is the $\ell_1$-ball, then the obtained bound scales with $\sqrt{\log(d)/n}$.




\subsection{Lower bounds on the excess risk}
Lower bounds for stochastic minimization of exp-concave functions have been studied in \cite{mahdavi2015lower}. For our setting, theorem 2 in this paper implies a bound of $\Omega(d/n)$ on the excess risk of any algorithm. 

For the special case of linear regression with 
\begin{equation} \label{eq:ohadSetting}
\cX = \{x \in \reals^d : \|x\|_2 \le 1\},~\cW = \{w \in \reals^d: \|w\|_2 \le B\},~\cY = [-Y,Y]
\end{equation}
\cite{shamir2014sample} proved the lower bound $\Omega \left( \min
  \{Y^2, \frac{B^2 + dY^2}{n}, \frac{BY}{\sqrt{n}} \} \right)$ on the
generalization error of ERM. The left-most term is trivially attained by the predictor $w=0$. The middle term is attained by combining the Vovk-Azoury-Warmuth forecaster (\cite{azoury2001relative,vovk2001competitive}) with standard online-to-batch conversions (\cite{cesa2004generalization}). Last, the right term is attained by ERM, as implied by the aforementioned upper bound of \cite{kakade2009complexity}. 

It is left open whether the middle term in the lower bound is attained by ERM. Note that if $B = \omega( \sqrt{d}Y)$, then the middle term in the above lower bound is asymptotically larger than our upper bound. However, since in the setting of \cite{shamir2014sample} (\eqref{eq:ohadSetting}) the magnitude of the predictions is not uniformly upper bounded by $Y$, no contradiction arises.




\subsection{Stability and Regularization}
Previous work (\cite{bousquet2002stability,shalev2010learnability}) studied the rate of uniform stability in various settings. For our setting, their bounds on the expected risk are identical to the bound in \lemref{lem:stabilityKappa}. As we explained above, these fast rates are often worse than the so-called slower rates due to the dependence on the empirical condition number. The standard approach for tackling this problem is add a regularization term. By adding the regularization term $\lambda \|w\|^2$ to the objective, one effectively increases the eigenvalues of $\hC$ by $\lambda$, and consequently, the overall condition number is decreased. However, as explained in \cite{shalev2014understanding}[Section 13.4], this modification usually does not preserve the fast rates.\footnote{Namely, when tuning $\lambda$, we need to ensure that any $\epsilon/2$-approximate minimizer with respect to the regularized objective is also an $\epsilon$-approximate minimizer with respect to the unregularized objective. As explained in \cite{shalev2014understanding}[Section 13.3], by optimally controlling this tradeoff, we no longer obtain fast rates (i.e., the stability rate scales with $1/\sqrt{n}$ rather than $1/n$).}

\subsection{Stability and Exp-concavity}
Informally, exp-concavity can be seen as a local and weaker form of strong convexity. Indeed, the Online Newton Step (ONS) of \cite{hazan2007logarithmic}, which has been designed for online minimization of exp-concave functions, achieves improved (logarithmic) regret bounds that resemble the regret bounds for strongly convex functions (\cite{hazan2007logarithmic}). The online-to-batch analysis of \cite{mahdavi2015lower} yields a bound on the excess risk that coincides with our bounds up to logarithmic factors. The main shortcoming of the ONS algorithm is that it employs expensive iterations (the runtime per iteration scales at least quadratically with $d$). Hence, it is natural to ask whether there exist simpler algorithms that achieve fast rates.

This question was answered affirmatively by \cite{koren2015fast}. This work, which is most closely related to our work, considers the minimization of a risk of the form $F(w) = \bE [f(w,Z)]$, where for any $z$, $f(\cdot,z)$ is $\bar{\beta}$-smooth\footnote{That is, the maximal eigenvalue of the Hessian of $f$ at any point $w$ is at most $\beta$.} and $\bar{\alpha}$-exp-concave function. They established fast rates for any algorithm that minimizes the \emph{regularized} risk $\hL(w) + \frac{1}{n}R(w)$, where $R(w)$ is assumed to be a $1$-strongly convex function (e.g., one can set $R(w)=\frac{1}{2} \|w\|^2$). While exp-concavity is weaker than strong convexity, \cite{koren2015fast}[section 4.2] interprets exp-concavity as strong convexity in the (local) norm induced by the outer products of the gradients and the regularization term. In other words, the problem is well-conditioned with respect to this local norm. Note that their formulation is more general in the sense that they do not assume a GLM structure. However, it should be emphasized that all the known exp-concave functions in machine learning are of the form (\ref{eq:risk})).

The above interpretation of \cite{koren2015fast} inspired us to make one step forward and directly show that regularization is not required as long as a related (preconditioned) problem is well conditioned. Besides the obvious importance of showing the insignificance of regularization in this context, we believe that the notion of preconditioned stability and its relation to the excess risk make these ideas more transparent and simplify the proofs.  

The upper bound of \cite{koren2015fast} on the excess risk scales with $\frac{24 \beta d}{\bar{\alpha} n}=\frac{24 \beta d \rho^2}{\alpha n}$ (recall that the exp-concavity parameter $\bar{\alpha}$ is equal to $\alpha/\rho^2$). Note that our analysis does not assume smoothness of the loss. This resolves the question raised by \cite{koren2015fast} regarding the necessity of the smoothness assumption. Note that for linear regression, the smoothness is $1$, making our bounds identical to the bounds of \cite{koren2015fast} for this special case. 

As discussed in \cite{koren2015fast}, it is difficult to translate bounds on the average stability into high-probability bounds (while preserving the fast rate and introducing only logarithmic dependence on $1/\delta$).

\subsection{Other Techniques and High-Probability Bounds} \label{sec:highProb}
The bound on the expected excess risk in \corref{cor:main} can be obtained by using two additional techniques. Both of these techniques also yield high probability bounds. We next survey the corresponding results. 

A recent follow-up work by \cite{mehta2016exp} established a bound of $\tilde{O}(d\log(n)+ \log(1/\delta)/n)$ on the excess risk of ERM, where $\delta$ is the confidence parameter.\footnote{The dependence on the exp-concavity parameter as well as the diameter of the loss function is hidden.}. He also showed how to get rid of the $\log(n)$ factor by boosting the confidence of ERM. The proof is centered around a Bernstein condition which holds due to the exp-concavity assumption.

Another alternative, is to bound the excess risk by the local Rademacher complexity (LRC) of the associated class of predictors (e.g., using Corollary 5.3 in \cite{bartlett2005local}). In our setting, one can derive bounds on the LRC (e.g., using \cite{koltchinskii2008oracle}) which coincide with our bounds. 

All of these techniques employ arguably heavy machinery and lack the geometric interpretation, which is nicely captured by our notion of preconditioned stability.


\section*{Acknowledgments}
We thank Iliya Tolstikhin for pointing out the alternative proof of \corref{cor:main} using local Rademacher complexities. 
\newpage
\bibliographystyle{plain}
\bibliography{bib}
\newpage 
\appendix


\end{document}